\documentclass[letterpaper, 10 pt, conference]{ieeeconf}  % Comment this line out if you need a4paper

\IEEEoverridecommandlockouts                              % This command is only needed if 
\usepackage{graphics} % for pdf, bitmapped graphics files
\usepackage{epsfig} % for postscript graphics files
\usepackage{times} % assumes new font selection scheme installed
\usepackage{algorithm,algorithmic}

\usepackage{amsmath}
\usepackage{esint}
\usepackage{amssymb}
\usepackage{amsthm}
\usepackage{mathtools}
\usepackage{derivative}
\usepackage{xcolor}
\usepackage{bm}
\usepackage{mathrsfs}
\usepackage{physics} 

\usepackage{graphicx}   % for \includegraphics
\usepackage{caption}    % (optional) for better control of captions
\usepackage{subcaption} % (optional) if you plan to do subfigures

\usepackage[noadjust]{cite}
\usepackage{url}
\usepackage[hidelinks]{hyperref}

\usepackage{graphicx}
\usepackage[export]{adjustbox} % For aligning figures
\graphicspath{{figures/}}
\usepackage{stfloats}
\usepackage{booktabs} % For professional rules (\toprule, \midrule, \bottomrule)
\usepackage{siunitx}  
\usepackage[inkscapelatex=false]{svg}
\usepackage{multirow}
\usepackage{diagbox}

\newtheorem{theorem}{Theorem}

\theoremstyle{definition}
\newtheorem{definition}{Definition}

\newtheorem{remark}{Remark}

\DeclareMathOperator*{\argmax}{arg\,max}
\DeclareMathOperator*{\argmin}{arg\,min}

\newcommand{\negone}{\scalebox{0.5}[1.0]{$-$} \hspace{-0.2mm} 1}

\newcommand{\bs}[1]{\boldsymbol{ #1 }}

\newcommand{\Sp}{\mathbb{S}}
\newcommand{\C}{\mathcal{C}}

\title{\LARGE \bf Layered Safety: Enhancing Autonomous Collision Avoidance \\ via Multistage CBF Safety Filters}
\author{Erina Yamaguchi, Ryan M. Bena, Gilbert Bahati, and Aaron D. Ames% <-this % stops a space
\thanks{This research is supported by BP.}% <-this % stops a space
\thanks{Authors are with the Department of Mechanical and Civil Engineering and the Department of Aerospace, Caltech, Pasadena, CA 91125, USA,  \{\tt\small erinay, ryanbena, gbahati, ames\}@caltech.edu}%
}

\begin{document}

\maketitle
\thispagestyle{empty}
\pagestyle{empty}

%%%%%%%%%%%%%%%%%%%%%%%%%%%%%%%%%%%%%%%%%%%%%%%%%%%%%%%%%%%%%%%%%%%%%%%%%%%%%%%%
\begin{abstract}
% Autonomous collision avoidance in unstructured dynamic environments is a challenging task, requiring robots to efficiently fuse real-time perception, planning, and control. % in the presence of sensor noise with limited onboard compute.
% %
% To tackle this challenge, we 
This paper presents a general end-to-end framework for constructing robust and reliable layered safety filters that can be leveraged to perform dynamic collision avoidance over a broad range of applications using only local perception data.
Given a robot-centric point cloud, we begin by constructing an occupancy map which is used to synthesize a \textit{Poisson safety function} (PSF). The resultant PSF is employed as a \textit{control barrier function} (CBF) within two distinct safety filtering stages.
In the first stage, we propose a predictive safety filter to compute optimal safe trajectories based on nominal potentially-unsafe commands. The resultant short-term plans are constrained to satisfy the CBF condition along a finite prediction horizon.
In the second stage, instantaneous velocity commands are further refined by a real-time CBF-based safety filter and tracked by the full-order low-level robot controller. 
Assuming accurate tracking of velocity commands, we obtain formal guarantees of safety for the full-order system. 
% Our method leverages \textit{reduced-order models}—simplified representations capturing dominant system dynamics—to enable generalizable safety synthesis via velocity commands while maintaining guarantees for the full-order system through accurate tracking. 
%
We validate the optimality and robustness of our multistage architecture, in comparison to traditional single-stage safety filters, via a detailed Pareto analysis. We further demonstrate the effectiveness and generality of our collision avoidance methodology on multiple legged robot platforms across a variety of real-world dynamic scenarios.

\end{abstract}

% {\color{blue}

% \begin{itemize}
%     \item "monolithic" MPC $\rightarrow$ perfect state knowledge is optimal
%     \item but in hardware...:
%     \begin{itemize}
%         \item "monolithic" MPC [bad because of non perfect perception]
%         \item safety filter [myopic behavior]
%         \item Multi-layer safety filter (MPC and real time safety-filter)
%     \end{itemize}
% \end{itemize}
% }

%%%%%%%%%%%%%%%%%%%%%%%%%%%%%%%%%%%%%%%%%%%%%%%%%%%%%%%%%%%%%%%%%%%%%%%%%%%%%%%%
\section{INTRODUCTION}
% {\color{blue}
% Talk about:
% \begin{itemize}
%     % \item Robot collision avoidance online, with environmental data coming directly from onboard sensors
%     % \item LiDAR denoising
%     % \item Safety from LiDAR: Learning-based, CBF-based Safety Filters, etc
%     \item Poisson Safety Functions (PSF)
%     \item Occupancy Maps
%     \item We extend PSF from RBG input to LiDAR input
% \end{itemize}
% }
As the deployment of robotic systems continues to evolve from controlled laboratories to unstructured real-world environments, autonomous safe operation in dynamic settings becomes crucial. For legged robots and humanoids, collisions can compromise task objectives, necessitating the constructive and robust integration of perception, planning, and control under sensor noise and partial observations. This is particularly challenging when considering the task of collision avoidance using onboard sensing and computation, where the robot must continually acquire environmental occupancy information, process it into a reliable safety representation, and translate this into safe control actions.
% Given these occupancy maps, Poisson Safety Functions \cite{bahati2025dynamic} provide a continuous representation of safety that can be used to generate safe control actions. 
%
% {\color{blue}

% % In perception-in-the-loop applications, sensor measurements provide sampled observations of the surrounding physical environment, which must be converted into a representation providing obstacle occupancy information.
% % %
% % Many occupancy mapping methods have been developed \cite{}, including probabilistic approaches for uncertain and dynamic environments\cite {}. 
% % These methods provide rich representations of occupancy.
% %
% % In the present work, occupancy mapping serves as the perception layer through which LiDAR sensor measurements are converted into a representation suitable for real-time safety-critical control.
% % but often address broader mapping objectives.. While the present work requires a representation tailored to real-time safety-critical control.

% % In this work, we use LIDAR point-cloud measurements to construct a computationally efficient perception interface for our proposed control architecture.. enabling the robot to continually acquire environmental
% % occupancy information, process it into a reliable safety
% % representation, and translate this into safe control actions.

% }
% This work aims to provide a modular framework for dynamic collision avoidance from LiDAR point clouds; integrating perception, safety function synthesis, trajectory generation and control into a unified pipeline.

\begin{figure}[t!]
    \centering   
    % \includegraphics[width=1\linewidth]%{example-image-a}
    % \vspace{-6mm}
    \includegraphics[width=\linewidth]{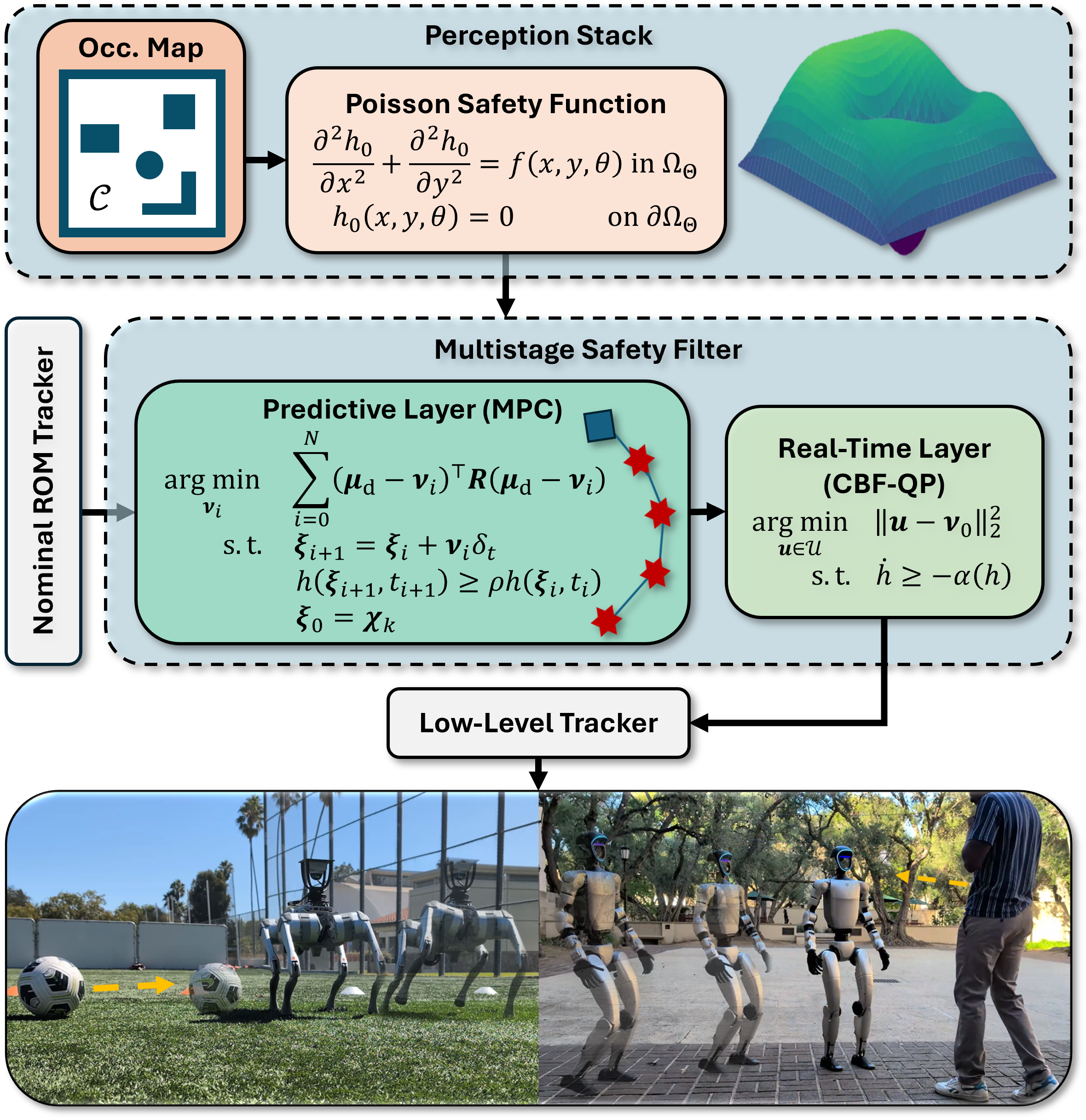}
    \caption{\footnotesize Layered Safety Filter. The PSF is computed using an occupancy map safe set representation. The solution enables the construction of safety constraints for use in a multistage safety filter: a predictive layer in series with a real-time layer. This framework generates safe autonomous actions, leading to collision-free behavior around dynamic obstacles. Experimental footage provided at https://youtu.be/8v7uyheOM2k}
    \label{fig:hero_figure}
    \vspace{-7mm}
\end{figure}

On this front, control barrier function (CBF) methods \cite{ADA-XX-JWG-PT:17,ADA-SC-ME-GN-KS-PT:19} provide a computationally efficient mechanism for enforcing real-time safety and have been demonstrated across many platforms \cite{tamas2022model, cohen2025safety, cohen2024safety, cohen2024constructive, bahati2025control, bahati2025dynamic}. However, most CBF-based methods operate myopically, relying on a quadratic program (QP) to optimize instantaneous safety, which can result in conservative behavior, nonpersistent feasibility, or task failure when immediate safe actions conflict with longer-term goals \cite{reis2020control, mestres2025control}. 
Model predictive control (MPC) \cite{borrelli2017predictive} mitigates this limitation by optimizing a future trajectory using safety-assured planning strategies \cite{ren_safety-assured_2025}, including enforcing discrete-time CBF constraints along a receding finite horizon \cite{zeng_dtcbfMpc_2021, ma2021feasibility, zeng2021enhancing, liu2022iterative, wabersich2022predictive, abdi2024model}. % MPC has demonstrated improved closed-loop feasibility and robustness compared to either method individually \cite{cosner2023robust}.
However, MPC exhibits fragility due to uncertainties about future environments and the challenges in solving highly nonconvex optimization problems. Furthermore, most existing safety-critical control methodologies handle the specification of safety constraints \textit{a priori}, typically through hand-designed geometric descriptions or signed distance fields \cite{long2021learning, jian2022dynamic, dawson2022learning, DeSa24, KeyumarsiRAL24, TooranjiACC25, yang2025shield, LiuGrizzle23}. 
In this paper, we address the real-time synthesis of safety constraints using the Poisson safety function (PSF) method \cite{bahati2025dynamic, bena2025geometry}, which produces smooth functional safety representations directly from perception data, enabling systematic CBF synthesis. Critically, the PSF approach produces a dynamically-updated unified collision avoidance constraint. 
% PSFs have been demonstrated for predictive safety filtering within MPC \cite{bena2025geometry}, but not within a unified multi-layer ROM-based architecture.

High-dimensional dynamics add complexity when ensuring safety. For legged robots, most approaches rely on full-order dynamics \cite{ChiuHutter22, gaertner2021collision}, typically rendering online planning and safety constraint synthesis computationally prohibitive and limiting reaction time. Furthermore, %full-order models may be imperfect or computationally intensive to evaluate, and 
in practice, safety filters often lack access to the full-order control inputs, only passing reference commands to low-level controllers. To overcome these limitations, recent approaches employ learning-based policies \cite{reflexiveBot} at the expense of formal safety guarantees. Instead, a layered architecture approach is introduced in \cite{grandia2021multi} combining low-frequency prediction with high-frequency safety enforcement. However, their formulation relies on knowledge of platform-specific dynamics with a custom low-level controller. To enhance generality, we leverage reduced-order models (ROMs) \cite{TamasRAL22, cohen2025safety, compton25a, cohen2024safety, esteban2025layered} that capture dominant dynamics while enabling tractable synthesis of safety constraints. This viewpoint decouples safety synthesis from full-order modeling: safe commands are planned on the ROM and safety for the full-order dynamics follows under accurate tracking of the ROM commands by the robot, yielding a \textit{model-free} architecture.

% {\color{blue}
% Building on these foundations, this work presents a multistage safety filter, embedded in an end-to-end perception-to-control architecture, that combines MPC-like predictive planning with CBF-QP real-time safety. This architecture is illustrated as a block diagram in Fig. \ref{fig:hero_figure}.

Building on these foundations, this work presents a multistage safety filter combining MPC-like predictive planning with CBF-QP real-time safety, embedded in an end-to-end perception-to-control architecture as illustrated in Fig. \ref{fig:hero_figure}. 
In the proposed model-free layered pipeline, a predictive safety filter generates safe nominal commands, that are then refined by a real-time CBF-QP safety filter.
This architecture unifies perception-based CBF synthesis with predictive and real-time CBF-based safety filtering, addressing a gap in existing approaches, typically emphasize only one part of this pipeline.
Existing occupancy mapping methods provide rich environment representations but lack constructive interfaces for real-time CBF synthesis. ROM-based CBF approaches enable computationally efficient real-time safety filtering \cite{TamasRAL22, cohen2025safety, compton25a, cohen2024safety, esteban2025layered}, but often do not incorporate perception or predictive planning. Predictive safety approaches account for long-horizon behavior, but become computationally expensive when combined with real-time safety under full-order dynamics \cite{grandia2021multi}.
% This architecture addresses limitations of existing perception-based CBF methods, which lack predictive planning, and ROM-based approaches, which typically employ either real-time filtering without prediction \cite{TamasRAL22, cohen2025safety, compton25a, cohen2024safety, esteban2025layered} or predictive planning with computationally expensive full-order real-time safety filtering \cite{grandia2021multi}.
%
%
%
%
Our contributions are summarized as follows:
\begin{itemize}
    \item \textit{Perception-based safety synthesis:} 
    % {\color{blue}A real-time point-cloud-based occupancy mapping algorithm from which Poisson safety functions synthesize CBF constraints with geometric parameterization.}
%
A real-time point-cloud-based occupancy mapping method for dynamic environments, enabling the synthesis of a PSF that serves as a CBF with geometric parameterization.
    % {\color{blue}
    % In the present work, occupancy mapping serves as the perception layer through which LiDAR sensor measurements are converted into a representation suitable for real-time safety-critical control.}
    \item \textit{Multistage model-free control architecture:} A cascaded MPC-based predictive safety filter and a real-time CBF-QP safety filter on a ROM, with theoretical guarantees of accurate tracking by the full-order robot dynamics.
    \item \textit{Performance analysis:} Characterization of performance optimality and robustness tradeoffs via a detailed Pareto optimality analysis, demonstrating the benefits of the layered approach over traditional single-stage methods.  
    \item \textit{Hardware verification \& validation:} Extensive demonstrations on quadrupedal and humanoid platforms performing dynamic collision avoidance in real-world environments with fully-onboard sensing and computation.
\end{itemize}

The structure of the paper is as follows: Section \ref{sect:background} reviews CBFs for safety-critical control, PSFs for perception-based safety, and MPC. Section \ref{sect:architecture} introduces the proposed perception and multistage control framework. The multistage architecture is validated experimentally in Section \ref{sect:experiment}. Finally, conclusions are in Section \ref{sect:conclusion}.

% {\color{blue}
% \noindent
% \textit{Related Work:}
% Existing approaches to dynamic collision avoidance from perception can be categorized by their architecture designs. Most work using CBFs with LiDAR \cite{long2021learning, jian2022dynamic, dawson2022learning, DeSa24, KeyumarsiRAL24, TooranjiACC25, yang2025shield, LiuGrizzle23} lack a predictive planning layer and thus, are limited in thier realtime capability in dynamic environments. ROM-based safety architectures \cite{TamasRAL22,cohen2025safety,compton25a,cohen2024rom,esteban2025layered}
% focus on real-time filtering without mid-level prediction, while \cite{grandia2021multi} uses ROM-based planning but retains the high-dimensional full-order model for low-level safety filtering. Our framework uniquely combines PSF-based safety synthesis \cite{bahati2025dynamic, bena2025geometry} with ROM-based MPC and real-time filtering in a unified perception-to-control pipeline for dynamic environments.
% }
% \newcommand{\setmap}[3]{#1:#2 \mathrel{\vcenter{\mathsurround0pt
% \ialign{##\crcr
% 		\noalign{\nointerlineskip}$\rightarrow$\crcr
% 		\noalign{\nointerlineskip}$\rightarrow$\crcr
% 		}}}%
% 		#3}

\newcommand{\naturals}{\mathbb{N}}
\newcommand{\re}{\mathbb{R}}
\newcommand{\R}{\mathbb{R}}
\newcommand{\realnonneg}{\mathbb{R}_{\ge 0}}
\newcommand{\realpos}{\mathbb{R}_{> 0}}
\newcommand{\until}[1]{[#1]}
\newcommand{\map}[3]{#1:#2 \rightarrow #3}
\newcommand{\qedA}{~\hfill \ensuremath{\square}}
\newcommand\scalemath[2]{\scalebox{#1}{\mbox{\ensuremath{\displaystyle #2}}}}
\newcommand{\interior}{\operatorname{int}}

\newcommand{\longthmtitle}[1]{\mbox{}{\textit{(#1):}}}
\newcommand{\setdef}[2]{\{#1 \; | \; #2\}}
\newcommand{\setdefb}[2]{\big\{#1 \; | \; #2\big\}}
\newcommand{\setdefB}[2]{\Big\{#1 \; | \; #2\Big\}}
\newcommand*{\SetSuchThat}[1][]{} % reserve the name
\newcommand*{\MvertSets}{%
    \renewcommand*\SetSuchThat[1][]{%
        \mathclose{}%
        \nonscript\;##1\vert\penalty\relpenalty\nonscript\;%
        \mathopen{}%
    }%
}
\MvertSets % default
% \DeclarePairedDelimiterX \Set [2] {\lbrace}{\rbrace}
%     {\,#1\SetSuchThat[\delimsize]#2\,}

\newcommand{\dt}{\mathrm{d}t}
\newcommand{\dy}{\mathrm{d}y}
\newcommand{\dx}{\mathrm{d}x}
\newcommand{\dtau}{\mathrm{d}\tau}
\newcommand{\Cc}{\mathcal{C}}
\newcommand{\Ac}{\mathcal{A}}
\newcommand{\pCc}{\partial \mathcal{C}}
\newcommand{\Bc}{\mathcal{B}}
\newcommand{\Tc}{\mathcal{T}}
\newcommand{\Dc}{\mathcal{D}}
\newcommand{\Kc}{\mathcal{K}}
\newcommand{\Oc}{\Omega}
\newcommand{\Occ}{\overline{\Omega}}
\newcommand{\pOc}{\partial \Omega}
\newcommand{\Ocext}{\Oc_\mathrm{ext}}
\newcommand{\Ocint}{\Oc_\mathrm{int}}
\newcommand{\Hc}{\mathcal{H}}
\newcommand{\Fc}{\mathcal{F}}
\newcommand{\Mc}{\mathcal{M}}
\newcommand{\Nc}{\mathcal{N}}
\newcommand{\Pc}{\mathcal{P}}
\newcommand{\Uc}{\mathcal{U}}
\newcommand{\Sc}{\mathcal{S}}
\newcommand{\Xc}{\mathcal{X}}
\newcommand{\Yc}{\mathcal{Y}}
\newcommand{\Vc}{\mathcal{V}}
\newcommand{\Zc}{\mathcal{Z}}
\newcommand{\Lc}{\mathcal{L}}
\newcommand{\Rm}{\mathcal{\mathbb{R}}}

\newcommand{\divv}{\nabla \cdot \vec{\bv}}
\newcommand{\hs}{h_\mathrm{\Sc}}

\newcommand{\defeq}{\triangleq}

\newcommand{\vr}{\varepsilon}
\newcommand{\nom}{{\operatorname{nom}}}
\newcommand{\des}{{\operatorname{des}}}
\newcommand{\on}{{\operatorname{on}}}
\newcommand{\off}{{\operatorname{off}}}
\newcommand{\fl}{{\operatorname{FL}}}
\newcommand{\Lie}{\mathcal{L}}
\newcommand{\qp}{{\operatorname{QP}}}

\newcommand{\ie}{i.e., }
\newcommand{\todo}[1]{{\color{cyan} Todo: #1}}

% Bold math shortcuts
% \newcommand{\ba}{\mathbf{a}}
\newcommand{\bb}{\mathbf{b}}
\newcommand{\be}{\mathbf{e}}
\renewcommand{\bf}{\mathbf{f}} % NOTE: REDEFINED \bf command
\newcommand{\bff}{\mathbf{f}}
\newcommand{\bg}{\mathbf{g}}
\newcommand{\bk}{\mathbf{k}}
\newcommand{\bp}{\mathbf{p}}
\newcommand{\bq}{\mathbf{q}}
\newcommand{\dq}{\dot{\mathbf{q}}}
\newcommand{\ddq}{\ddot{\mathbf{q}}}
\newcommand{\bu}{\mathbf{u}}
\newcommand{\bv}{\mathbf{v}}
\newcommand{\bvv}{\vec{\mathbf{v}}}
\newcommand{\bn}{\mathbf{n}}
\newcommand{\hbn}{\hat{\mathbf{n}}}

\newcommand{\bx}{\mathbf{x}}
\newcommand{\bz}{\mathbf{z}}
\newcommand{\br}{\mathbf{r}}
\newcommand{\bA}{\mathbf{A}}
\newcommand{\bB}{\mathbf{B}}
\newcommand{\bD}{\mathbf{D}}
\newcommand{\bC}{\mathbf{C}}
\newcommand{\bF}{\mathbf{F}}
\newcommand{\bJ}{\mathbf{J}}
\newcommand{\bG}{\mathbf{G}}
\newcommand{\bK}{\mathbf{K}}
\newcommand{\bP}{\mathbf{P}}
\newcommand{\bW}{\mathbf{W}}
\newcommand{\bw}{\mathbf{w}}
\newcommand{\bd}{\mathbf{d}}
\newcommand{\bvy}{\vec{\by}}
\newcommand{\bty}{\tilde{\by}}
\newcommand{\bbeta}{\boldsymbol{\eta}}
\newcommand{\mb}[1]{\mathbf{#1}}

\newcommand{\bY}{\mathbf{Y}}
\newcommand{\by}{\mathbf{y}}
\newcommand{\byobs}{\mathbf{y}_\mathrm{obs}}
\newcommand{\bl}{\mathbf{\lambda}}

\newcommand{\bxd}{\bx_\mathrm{d}}
\newcommand{\bxobs}{\bx_\mathrm{obs}}
\newcommand{\md}{\mathrm{d}}

\newcommand{\Uxd}{U_{\mathrm{d}}}
\newcommand{\Uobs}{U_{\mathrm{obs}}}
\newcommand{\Uapf}{U_{\mathrm{APF}}}

\newcommand{\GradUxd}{\nabla U_{\mathrm{d}}}
\newcommand{\GradUobs}{\nabla U_{\mathrm{obs}}}
\newcommand{\GradUapf}{\nabla U_{\mathrm{APF}}}

% \SetKwComment{Comment}{/* }{ */}
\newcommand{\cmax}{c_\mathrm{max}}
\newcommand{\cmin}{c_\mathrm{min}}

\newcommand{\bch}{\bs{\chi}}

% \clearpage
\section{Background}\label{sect:background}
%We begin by reviewing the framework for safety-critical control based on CBFs. 
We consider nonlinear control affine systems of the form:
\begin{align}\label{eq: nl system}
    \dot{\bx} = \bf(\bx) + \bg(\bx)\bu,
\end{align}
where $\bx \in \re^n$ denotes the state vector, $\bu \in \re^m$ represents the control input, $\bf:\re^n \rightarrow \re^n$ denotes the drift dynamics, and $\bg:\re^n \rightarrow \re^{n \times m}$ is the actuation matrix, both assumed to be locally Lipschitz continuous. A locally Lipschitz continuous controller, $\bk:\re^n \times \re_+\rightarrow \re^m $, yields the closed loop system $\dot{\bx} = \bf(\bx) + \bg(\bx)\bk(\bx,t)$ which, given an initial condition, admits a unique solution $t\mapsto \bx(t)$ 
which we assume exists for all time for brevity \cite{perko2013differential}.

\subsection{Safety-Critical Control}
We formalize the notation of safety through the lens of \textit{forward invariance}, where system trajectories $t\mapsto \bx(t)$ are required to remain in a \textit{ safe set }for all time. We consider time-varying safe sets defined as the $0$-super level set of a function  $h:\re^n \times \re_{+} \rightarrow \re$ denoted by:
\begin{align}
    \Cc_t = \{\bx \in \re^n: h(\bx,t) \geq 0 \}.
\end{align}
CBFs provide a constructive framework for designing controllers that enforce forward invariance.

\begin{definition}(Control Barrier Function~\cite{ADA-XX-JWG-PT:17}) Let $\Cc_t \subset \re^n$ be the time-varying $0$-super-level set of a continuously differentiable function  $h:\re^n \times \re_+ \rightarrow \re$ satisfying $\nabla h(\bx,t) \neq \mathbf{0}$ when $h(\bx,t) =0$. We call $h$ a time-varying CBF for \eqref{eq: nl system} if there exists\footnote{\small A continuous function $\alpha : \R \to \R$ belongs to the extended class $\Kc_{\infty}^{e}$ if it is monotonically increasing, $\alpha(0) = 0$, $\lim_{s \to \infty} \alpha(s) = \infty$, and $\lim_{s \to -\infty} \alpha(s) = -\infty$.}
 $\alpha \in \Kc_{\infty}^{e}$ such that for all $(\bx,t) \in \re^n \times \re_+$:
\begin{align}\label{CBF condition}
    \sup_{\bu \in \re^m} \dot h(\bx,t,\bu) \coloneq \nabla &h (\bx,t) \cdot (\bf(\bx) + \bg(\bx) \bu) \nonumber \\& + \frac{\partial h}{\partial t}(\bx,t) \geq - \alpha(h(\bx,t)).
\end{align}
% 
% \begin{align}\label{CBF condition}
%     h(\bx,t,\bu)  = \nabla h (\bx,t) \cdot (\bf(\bx) + \bg(\bx) \bu) + \frac{\partial h}{\partial t}(\bx,t)
% \end{align}
% holds for all $\bx \in \re^n$.
\end{definition}

Given a nominal (possibly unsafe) controller $\bk_\mathrm{nom}:\re^n \times \re_+ \rightarrow \re^m$, one way of constructing a safe controller is via the CBF-QP that adjusts $\bk_\mathrm{nom}$ to its nearest safe action:
% the quadratic programming-based controller that adjusts $\bk_\mathrm{nom}$ to its nearest safe action through the nominal controller:
%
\begin{align*}\label{eq: safety filter}
    \bk(\bx,t) = &\argmin_{\bu \in \re^m} &&\|\bu - \bk_{\mathrm{nom}}(\bx,t)\|_2^2 \tag{CBF-QP} \\
    % &\quad \quad \quad \quad \ \\
    & \quad 
 \mathrm{s.t.} && \dot h(\bx,t,\bu)   \geq - \alpha(h(\bx,t)).
\end{align*}
%
% \subsection{Outputs and Relative Degree}
% We focus on systems for which safety specifications are expressed via desired  \textit{outputs} \cite{isidori1985nonlinear}. In particular, we consider safety specifications described in spatial coordinates:
% \begin{align}\label{eq: spatial coordinates}
% \by(\bx) = (x,y,z) \in \re^3   
% \end{align}
% %
% Under appropriate relative degree assumptions, methods such as \cite{bahati2025dynamic, bahati2025control, cohen2024constructive} provide constructive techniques for designing CBFs for systems \eqref{eq: nl system} with outputs \eqref{eq: spatial coordinates}. 
%
% The spatial coordinates \eqref{eq: spatial coordinates} provide a setting enabling perception-based CBF synthesis methods \cite{bahati2025dynamic}, which we discuss next. 
% methods to 
% The next subsection discusses the synthesis of CBFs for safety specifications \eqref{eq: spatial coordinates} directly from perception data, as developed in \cite{bahati2025dynamic}.
% {\color{blue} Discuss reduced-order-models..
% }

% Building on this foundation, we present a method for synthesizing safe sets and corresponding CBFs for spatial safety specifications directly from perception data, as introduced in \cite{bahati2025dynamic}.

\subsection{Perception-based Safety with Poisson Safety Functions}

While CBFs provide a principled tool for synthesizing controllers that guarantee safe actions, their success in robotic applications rely on the availability of a function $h_0$ to encode physical environmental safety specifications from perception data in spatial coordinates $\by(\bx) = (x,y,z) \in \re^3$. 
%
% -based safety specifications  defining a desired safe set $\Cc_0 \subset \re^3$.
%
% reliable environment representations derived from perception-based safety specifications  in spatial coordinates $\by(\bx) = (x,y,z) \in \re^3$ defining a desired safe set $\Cc_0 \subset \re^3$.
% in the form of occupancy information. 
%
Given occupancy information from perception, PSFs \cite{bahati2025dynamic} provide such safety representations, bridging the gap between perception and CBF-based safety.

%
% In collision avoidance applications, safety specifications are often inferred from  environmental perception data in the form of occupancy information via sensors, characterizing regions of space that are occupied or free in spatial coordinates $\by(\bx) = (x,y,z) \in \re^3  $.
%
%
% without an explicit functional representation. 
% %
% More often, general occupancy information for the local environment is collected in realtime, either via onboard sensors or global perception systems.
% In this work, we focus on systems for which safety specifications are expressed via desired  \textit{outputs} \cite{isidori1985nonlinear} described in spatial coordinates:
% \begin{align}\label{eq: spatial coordinates}
% \by(\bx) = (x,y,z) \in \re^3   
% \end{align}
%
%
% Given occupancy data, PSFs provide a functional representation of safety by solving a boundary value problem for Poisson's equation.

% In particular, given an occupancy map, let $\Oc \in \re^3$ be an open, bounded and connected set defining free space, and $\pOc$ represent obstacle surfaces. A PSF $h_0:\re^3 \rightarrow \re$ is defined as the unique solution to the following Dirichlet problem for Poisson's equation:

\begin{definition}(Poisson Safety Function \cite{bahati2025dynamic})
Given an occupancy map, let $\Oc \subset \re^3$ denote the open, bounded, and connected set of free space, with smooth boundary $\pOc$ corresponding to obstacle surfaces. A PSF, $h_0:\re^3 \rightarrow \re$, is defined as the unique solution Dirichlet problem for Poisson's equation:
\begin{gather}\label{eq: poisson's eq}
\left \{
    \begin{aligned}
        \Delta h_0(\by) &= f(\by)& \text{ in } \Omega,\\
        h_0(\by) &= 0 &  \text{ on } \partial \Omega, \\
    \end{aligned}
    \right.
\end{gather}
where \( \Delta = \frac{\partial^2 }{\partial x^2} + \frac{\partial^2 }{\partial y^2} + \frac{\partial^2 }{\partial z^2}\) is the \textit{Laplacian} and $f: \Oc \rightarrow \re_{<0}$ is a given forcing function. 
\end{definition}
The boundary conditions in \eqref{eq: poisson's eq} assigns the $0$-level set of $h_0$ to obstacle surfaces, while $f(\by) <0$ guarantees $h_0(\by) > 0$ throughout the free space. As discussed in \cite{bahati2025dynamic}, a smooth forcing function $f \in C^\infty(\Occ)$, yields smooth solutions to \eqref{eq: poisson's eq}, $h_0 \in C^\infty(\Occ)$. It follows that $h_0$ defines a CBF for systems with single-integrator dynamics (\ie $h_0 \coloneqq h$), and its smoothness allows for CBF extensions for higher-order systems  \cite{bahati2025dynamic}. For dynamic environments, \eqref{eq: poisson's eq} can be parametrized with time to account for temporal variations yielding a time-varying safety function $t \mapsto h_0(\by,t)$ \cite{bena2025geometry}.

\subsection{Model Predictive Control}

Controllers based on predictive optimization enable the parallel enforcement of both safety and performance goals. In our collision avoidance framework, we employ an MPC-based \emph{predictive safety filter} that incorporates the discrete-time CBF (DCBF) as in \cite{zeng_dtcbfMpc_2021,cosner2025dynamic}, enforcing safety over a time horizon. The following finite-time optimal control problem (FTOCP) with a DCBF constraint along a horizon of length $N \in \mathbb{N}_{\geq 1 }$ defines the predictive safety filter:
\begin{align}
    \underset{\bs{\nu}_{0:N-1} \in \R^m }{\underset{\bs{\xi}_{0:N} \in \R^n }{\min }} & \quad \sum_{i=0}^{N-1} c(\bs{\xi}_i, \bs{\nu}_i) + V(\bs{\xi}_N) \\
    \textup{s.t. } & \quad \bs{\xi}_{i+1} = \mb{F}(\bs{\xi}_i, \bs{\nu}_i), \quad \forall i \in \{ 0, \dots, N-1\}  \nonumber\\
    & \quad h(\bs{\xi}_{i+1}) \geq \rho h(\bs{\xi}_i), \quad \forall i \in \{ 0, \dots, N-1\} \nonumber\\
    % & \quad \bs{\nu}_{i} \in \mathcal{U}, \quad \forall i \in \{ 0, \dots, N-1\} \nonumber \\
    & \quad \bs{\xi}_0 = \bs{x}_k \nonumber 
\end{align}

\noindent where $c: \R^n \times \R^m \to \R$ is the stage cost, $V: \R^n \to \R$ is the terminal cost used to approximate the infinite-horizon optimal control problem, $\bs{\xi}_i \in \R^n $, $\bs{\nu}_i \in \R^m $ represent the planned state and inputs at discrete times, and $\bF: \re^n \times \re^m \rightarrow \re^n$ is the discretization of \eqref{eq: nl system} over a finite time interval. The first constraint ensures the planned trajectory is dynamically feasible, the second constraint is the DCBF condition that enforces safety along the trajectory, and the last constraint ensures the plan aligns with the current state. 

% To generate an input using this FTOCP, 
The optimal plan of inputs of this FTOCP is computed as $[\bs{\nu}_0^*(\bs{x}_k), \dots, \bs{\nu}_{N-1}^*(\bs{x}_k) ]$, and $\bs{\nu}_0^*(\bs{x}_k)$ is applied to the system, yielding the MPC+CBF controller: 
\begin{align}
    \bs{\pi}^{\textup{MPC+CBF}}(\bs{x}_k) = \bs{\nu}_0^*(\bs{x}_k).  
\end{align}
This $\bs{\pi}^\textup{MPC+CBF}$ controller displays favorable performance and robustness properties when compared to standard MPC or CBF controllers \cite{cosner2025dynamic} by enforcing the safety constraint in the FTOCP and optimizing across the planning horizon $N$.

% \clearpage

% \subsection{Reduced Order Models}
% We consider robotic systems with generalized coordinates $\bq \in \re^d$ and associated velocities $\dot{\bq} \in \re^d$ with dynamics:
% %
% \begin{align}
% \label{eq: mechanical system}
% M(\bq)\ddot{\bq} + C(\bq, \dot{\bq})\dot{\bq} +G(\bq) = \bu
% \end{align}
% %
% where $M(\bq) \in \re^{d \times d}$, $C(\bq, \dot{\bq})\in \re^{d \times d}$, $G(\bq)\in \re^{d}$  denote the inertial, coriolis matrices, and gravity terms respectively. 

% Constructing CBFs for the above (typically) high-dimensional system can be challenging.
% % since candidate barriers must reflect the behavior of the full-order dynamics to guarantee safety.
% %
% In practice, full-order models may be imperfect or computationally intensive to evaluate, which limits their use in real-time controllers. In many settings one also lacks direct access to the full-order control input and can only pass reference commands to modules within the autonomy stack that compute those inputs.
% %
% We address these challenges by exploiting a layered architecture and interpreting safety-critical control as certifying the ability of the full-order system to track inputs generated by suitably designed reduced-order models (ROMs) \cite{TamasRAL22,cohen2025safety,compton25a,cohen2024rom}. This viewpoint decouples safety synthesis from full-order modeling and yields a \textit{model-free} safety layer. Safety for the full-order dynamics then follows provided those dynamics track the ROM commands. 

\section{Multistage Safety Filter: Architecture}\label{sect:architecture}% In this section, we present a multi-stage safety framework for dynamic collision avoidance in robotic systems that integrates point-cloud based perception with Poisson-based safety representations, and constructive CBF synthesis with predictive and real-time based safety filtering. }

In this section, we present a multistage safety framework that unifies onboard perception, CBF synthesis, and MPC for dynamic collision avoidance in robotic systems.
Specifically, we build an occupancy map from point cloud data, construct a boundary value problem for Poisson's equation, and solve it to obtain a PSF. The resulting PSF serves as a CBF that is incorporated into both a \textit{predictive} safety layer---enforcing safety over a receding horizon via MPC---and a \textit{real-time} safety layer---enforcing instantaneous safety via a CBF-QP.

% MPC for real-time safety-critical control.

% In this section, we develop the constructive perception-to-control pipeline underlying the proposed layered safety filter, integrating robot-centric point-cloud perception, Poisson safety function synthesis, and predictive CBF-based control for real-time collision avoidance in dynamic environments. 

% The pipeline begins with an occupancy representation obtained from local point-cloud measurements, solves Poisson's equation to obtain a smooth safety representation, and uses the resulting Poisson safety function to construct geometry-parameterized CBF constraints for the predictive and real-time filtering layers.

\subsection{Occupancy Mapping}
\label{sec: Occupancy}

\begin{figure*} [t]
    \centering    \includegraphics[width=1\linewidth]{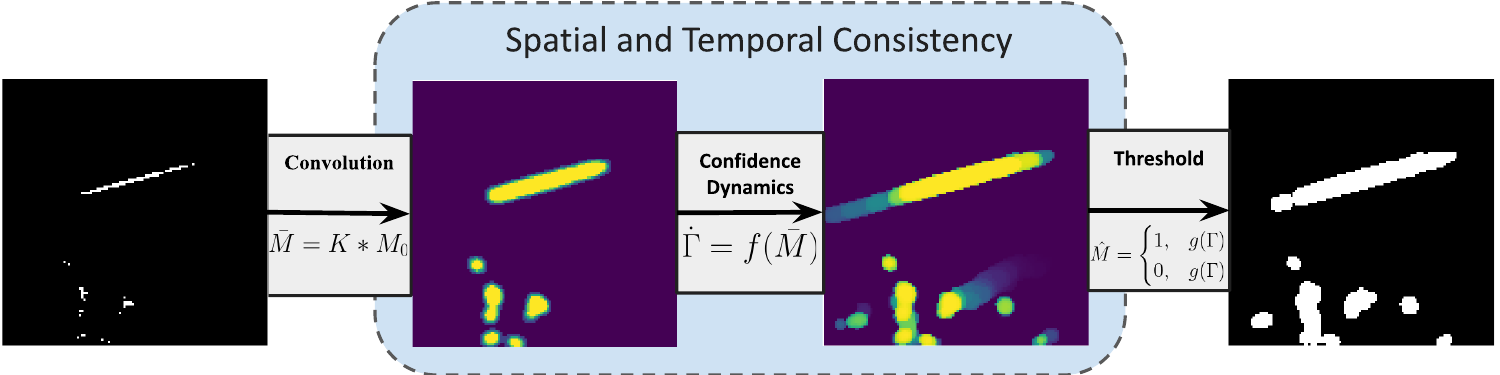}
    \vspace{-3mm}
    \caption{\footnotesize Occupancy Mapping. \textbf{Left} First, we project each point cloud onto a 2D grid $M_0$. \textbf{Middle Left} Second, we convolve $M_0$ with a Gaussian kernel $K$, adding weight to occupied clusters and producing the map $\bar{M}$. \textbf{Middle Right} Third, we use $\bar{M}$ to dynamically update the confidence map $\Gamma$. \textbf{Right} Last, we threshold $\Gamma$ with hysteresis to generate a binary occupancy map $\hat{M}$.}
    \label{fig: OccupancyMapping}
    \vspace{-6mm}
\end{figure*} 

Robotic systems collect local environment occupancy information in real-time via onboard sensors. This information is used to estimate
membership of a safe set, denoted as $\C_0\subset\R^2$, containing all unoccupied points $(x,y)\in\R^2$ in local two-dimensional environments. For perception-in-the-loop applications, the true set $\C_0$ cannot be known exactly; however, a discretized approximation of $\C_0$ can often be obtained in the form of an occupancy map. 

Let $\mathcal{S}$ be the set of grid cells in a local 2D occupancy map, and let $p \in \mathcal{S}$ be a single grid cell in the map corresponding to a fixed location in space. The occupancy map representing the true safe set $\C_0$ is the binary function $M : \mathcal{S} \to \{0,1\}$:
\begin{align}
    M(p) = \begin{cases}
        1, & \text{if cell $p \in \Sc$ is occupied}, \\[4pt]
        0, & \text{if cell $p \in \Sc$ is free}.
    \end{cases}
\end{align}
Next, we compute the occupancy map $\hat{M}$, which estimates $M$ from successive robot-centric point clouds.

\subsubsection{Instantaneous Point Map ($M_0$)}
\label{sec:projected2d}
To construct $\hat{M}$, we first remove ground and self detections, crop unmapped regions, and project each remaining detection onto the instantaneous 2D occupancy map $M_0 : \mathcal{S} \to \{0,1\}$ via:
\begin{align}
    \!M_0(p) := \begin{cases}
        1,\!& \text{if $p \in \Sc$ contains a point detection}, \\[4pt]
        0, \!& \text{otherwise}.
    \end{cases}
\end{align} 
\subsubsection{Spatially Convolved Map ($\bar{M}$)}

Generally, $M_0$ contains gaps and isolated cells that arise from noise, occlusion, and the limited angular resolution of perception-sensor scans. To improve spatial consistency, we convolve $M_0$ with an unnormalized compactly supported\footnote{A kernel is \textit{compactly supported} if it vanishes outside a finite radius $r$, so only neighboring cells within distance $r$ contribute to the convolution.} 
kernel $K$:
%$K: \mathcal{S}\times\mathcal{S} \rightarrow [0,1]$, defined as:
%
%\begin{align}\label{eq: Kernel}
%    K(p,q) = \kappa(d(p,q)),
%\end{align}
%where $p,q \in \mathcal{S}$ are grid cell indices, $d$ is a distance function, and 
%$\kappa:[0,\infty)\to[0,1]$ is a monotonically decreasing function with $\kappa(0)=1$.
%
\begin{equation}
    \bar{M} := K * M_0,
\end{equation}

\noindent where $*$ denotes discrete convolution operator. This operation smooths $M_0$ and amplifies regions of spatial consistency.
%If an occupied cell $p$ has no occupied neighbors, the convolution leaves $\bar M(p) = M_0(p) = 1$, while surrounding cells remain close to $0$. Conversely, if neighboring cells are occupied, $\bar M(p) > 1$ from summed contributions of other cells.
%Applying this kernel to $M_0$ yields the convolved occupancy grid $\bar{M} = K * M_0: \Sc \rightarrow \re_{\geq 0}$ where $*$ denotes discrete convolution:
%
%\begin{align}\label{eq:conv_int}
%    \bar M(p) = \sum_{q \in \mathcal{S}} K(p,q)\,M_0(q).
%\end{align}
% which clusters nearby occupied cells, fills small gaps from sparse measurements, and suppresses the effect of isolated measurements.
%
%If an occupied cell $p_0 \in \Sc$ has no occupied neighbors, convolution leaves $\bar M(p_0) = \tilde M(p_0) =1$, while surrounding cells remain close to $0$. If neighboring cells are occupied, $\bar M(p_0) > 1$ from summed contributions her cells. Thus, thresholdin$\bar{M}[p] > \tau > 1$) yields spatiathe ) > 1$ from summed con with lly consistent obstacle estimates: clusters are preserved, while isolated detections are filtered as $\tau$ increases.

\subsubsection{Dynamic Confidence Map ($\Gamma$)}
\label{sec:temporal_dynamics}
To enforce temporal continuity, we define a continuous \textit{confidence map }$\Gamma : \mathcal{S} \to [0,1]$, where $\Gamma(p)$ represents the confidence a cell $p \in \Sc$ is occupied. Given $\bar{M}$, the confidence dynamics are governed by the cell-wise continuous-time update law: 
\begin{align}
\label{eq: continuous_confidence}
   \dot{\Gamma}(p) := 
   \begin{cases}
       -\beta^-\Gamma(p) &\text{if } \bar{M}(p) < \sigma, \\
       \beta^+\bar{M}(p)\left(1-\Gamma(p)\right) &\text{if } \bar{M}(p) \geq \sigma,
   \end{cases}
\end{align}

\noindent 
% where $\beta^-,\beta^+ \in \re_{>0}$ are tunable gains representing nominal first-order parameters. 
% For ``free cells" the decay rate $\beta^-$ is fixed; meanwhile, for ``occupied cells", $\beta^+$ is augmented by the local value\footnote{This value corresponds to evidence of true occupancy in $M$.} of $\bar{M}$, producing faster growth of $\Gamma(p)$. 
%
with initial confidence $\Gamma_0(p)\in [0,1]$.
Here, $\beta^-,\beta^+\in\mathbb{R}_{>0}$ are tunable gains. When $\bar{M}(p)<\sigma$, $\Gamma(p)$ decays exponentially toward $0$ at rate $\beta^-$. When $\bar{M}(p)\ge\sigma$, $\Gamma(p)$ increases toward $1$, with growth scaled by $\beta^+\bar{M}(p)$, accelerating confidence accumulation in regions with stronger local evidence.
The switching parameter $\sigma \in \re_{>0}$ sets the level of spatial consistency required before confidence increases, emphasizing clusters while de-emphasizing isolated detections.

\subsubsection{Estimated Occupancy Map ($\hat{M}$)}
Binary thresholding is a natural choice for transforming the confidence map $\Gamma$ into the estimated occupancy map $\hat{M}: \Sc \rightarrow \{0,1\}$; however, confidence fluctuations near a threshold produce undesirable boundary noise. 
To mitigate this, we adopt \textit{hysteresis} thresholding, which accounts for the previous occupancy state:
\begin{align}\label{eq: hysteresis}
    \hat{M}(p) := \begin{cases}
        1 & \text{if } \Gamma(p) \geq \tau_h,\\
        \hat{M}_0(p) &\text{if } \tau_{\ell} < \Gamma(p) < \tau_h, \\
        0 & \text{if }  \Gamma(p) \leq \tau_{\ell},
    \end{cases}
\end{align}
\noindent where $\tau_h, \tau_\ell \in [0,1]$, $\tau_h > \tau_\ell$ are the upper and lower thresholds respectively, and $\hat{M}_0(p)$ denotes the previous occupancy estimate. 
This maintains the current occupancy state when confidence lies within $(\tau_h, \tau_\ell)$, requiring confidence to cross $\tau_h$ to become \textit{occupied} or fall below $\tau_\ell$ to become \textit{free}.
The resulting occupancy map $\hat{M}$ provides a spatially and temporally consistent estimate of the true occupancy $M$. 
A visual of the entire mapping procedure is depicted in Fig.\;\ref{fig: OccupancyMapping}.

% Map + Poisson
\subsection{Geometry-Aware Poisson Safety Function}
\label{sec: Poisson}

Consider the safe set $\C_0$ (and its discretized estimate $\hat{M}$). To enforce the forward invariance of $\C_0$, we generate a PSF \cite {bena2025geometry}. We begin by employing a rigid body ROM in which the robot is assumed to have a 2D centroidal position $(x, y)\in\re^2$ and orientation $\theta\in\Sp^1$. We parameterize the safe set in orientation using the Pontryagin difference:
\begin{equation}
    \C(\theta) = \C_0 \ominus \mathcal{R}(\theta),
\end{equation}
\noindent where $\mathcal{R}(\theta)\subset\R^2$ is the set of all points occupied by the robot, in body-centered coordinates, as a function of $\theta$. 

We aim to find a safety function $h_0 : \R^2\times\Sp^1 \rightarrow \R$, whose $0$-superlevel set characterizes $\C(\theta)\subset\R^2$ via:
\begin{equation}
\label{eq: Safe Set}
    \C(\theta) = \{(x,y)\in\R^2 : h_0(x,y,\theta) \geq 0 \}.
\end{equation}
\noindent We use a parameterized safe set $\C(\theta)$ to define an orientation-dependent domain $\Oc_{\Theta}$ via the lifting operation:
\begin{align}
    \!\!\!\Occ_{\Theta} \!= \!\bigcup_{\theta\in\Sp^1} \C(\theta) \times \{\theta\} \subset\R^2\times\Sp^1,
\end{align}
\noindent and formulate the Dirichlet problem for Poisson's equation:
\begin{equation}
\label{eq: Q Poisson}
\left\{
    \begin{aligned}
        \frac{\partial^2 h_0}{\partial x^2} + \frac{\partial^2 h_0}{\partial y^2} &= f(x,y,\theta) \quad \!\!\!\!&\forall(x,y,\theta)\in\Oc_{\Theta}, \\
        h_0(x,y,\theta) &= 0 \quad \!\!\!\!&\forall(x,y,\theta)\in\partial\Oc_{\Theta}.
    \end{aligned}
\right.
\end{equation}

\noindent 
Solving \eqref{eq: Q Poisson} yields a PSF $h_0$ that characterizes $\C(\theta)$ in according to \eqref{eq: Safe Set}, and renders $\C(\theta)$ forward invariant with respect to the dynamics describing $(x,y, \theta)$. 

Table \ref{tab:occupancy_comparison} compares our \textit{custom} mapping algorithm with two baselines: the instantaneous point map (Sec. \ref{sec:projected2d}), and the thresholded confidence map (Sec. \ref{sec:temporal_dynamics}). The former is generated by projecting raw measurements onto a 2D grid, naively treating measurement as ground truth. The latter represents a standard cost map, where cells track occupancy confidence which is updated using new measurements. The comparative baselines are therefore denoted \textit{naive} and \textit{standard}, respectively. In the \textit{static} scenario, each algorithm mapped an unchanging scene. In the \textit{dynamic} scenario, a ball was rolled through the environment. From both experiments, it is clear that, although the dynamic confidence update reduces noise, it does not effectively capture dynamic or sparsely-detected obstacles without spatial smoothing and hysteresis thresholding. Only our approach sufficiently de-noises the scene while retaining high ground-truth accuracy.

\begin{table}[h]
\centering
\caption{\footnotesize Comparison of mapping methods. \textit{RMS Noise} (static environment only) was computed using the RMS value of $\frac{\partial h}{\partial t}$ over all grid cells and frame transitions. $\frac{\partial h}{\partial t}=0$ in static scenes. \textit{Coverage \%} is the percent of ground-truth occupied cells, obtained from an overhead video, that are occupied in the map. For each scenario, the best noise/coverage are in \textbf{bold}.}
\label{tab:occupancy_comparison}
\scriptsize
\setlength{\tabcolsep}{5pt}

\resizebox{\columnwidth}{!}{
\begin{tabular}{l cc cc cc}
\toprule
\multirow{2}{*}{\textbf{Scenario}}
& \multicolumn{2}{c}{\textbf{Naive} (Projection)}
& \multicolumn{2}{c}{\textbf{Standard} (Confidence Map)}
& \multicolumn{2}{c}{\textbf{Custom} (Full Algorithm)} \\
\cmidrule(lr){2-3} \cmidrule(lr){4-5} \cmidrule(lr){6-7}
& \textbf{RMS Noise} & \textbf{Coverage \%}
& \textbf{RMS Noise} & \textbf{Coverage \%}
& \textbf{RMS Noise} & \textbf{Coverage \%} \\
\midrule
Static
& 0.249 & 47.40
& 0.181 & 39.45
& \textbf{0.118} & \textbf{100} \\

Dynamic
& -- & 10.93
& --  & 24.61
& -- & \textbf{97.79} \\
\bottomrule
\end{tabular}
}
\end{table}

\subsection{Predictive Safety Layer}
\label{sec: Controller}

In our predictive layer, we eliminate the need for a dynamic model by planning in the task space. Specifically, we plan safe velocity commands for the centroid configuration $\boldsymbol{\chi}=(x,y,\theta)\in\mathbb{R}^2\times\mathbb{S}^1$ governed by the integrator ROM:
% We eliminate the need for a custom dynamic model by leveraging a single-integrator reduced-order model (ROM) with state $\boldsymbol{\chi} =(x,y,\theta) \in\mathbb{R}^2\times\mathbb{S}^1$ consisting of the planar position and orientation of the robot's centroid, with dynamics:
%
% %
% We seek to design a safe controller by synthesizing and tracking a safe velocity $\dot{\bq}_s$. 
% This model considers the system outputs to be the fully-actuated system states $\bs{\chi}=[x,y,\theta]^\top\in\R^2\times\Sp^1$, resulting in the continuous-time dynamic model: 
%
\begin{equation}
    \label{eq: single integrator}
    \dot{\bs{\chi}}=\bs{\mu}, \qquad \bs{\mu}=(v_x,v_y,\omega)\in\R^3,
    % \dot{\bq}=\bs{\mu}
\end{equation}
\noindent
and rely on the low-level controller to track velocity references on the full-order model \cite{tamas2022model,cohen2025safety, cohen2024safety, compton25a}.
% We seek to design a safe controller by synthesizing and tracking a safe velocity $\dot{\bq}_s$. 
%
% Given environmental perception information, let $h_0:\mathbb{R}^2 \times [0,T] \to\mathbb{R}$ be a continuously differentiable, time-varying safety function whose $0$-superlevel set defines safe set:
% \[
% \mathcal{C}_t \;=\; \{\mathbf{q}\in\mathbb{R}^d \mid h_0(\mathbf{q},t) \ge 0\}.
% \]
%
Under these assumptions, the PSF $h_0$ is a CBF for the safe set $\C(\theta)$.

Given a nominal command $\bs{\mu}_\mathrm{d}$, we introduce a predictive safety planner for \eqref{eq: single integrator} which aims to enforce the forward invariance of a time varying set $\C'(\theta,t)\subset\re^2$ over a finite time horizon using DCBF constraints, while simultaneously minimizing the cumulative deviation from $\bs{\mu}_\mathrm{d}$: 
\begin{align}
\label{eq: NMPC}
    \underset{\bs{\nu}_i}{\min } & \quad \sum_{i=0}^N \left(\bs{\mu}_\mathrm{d}-\bs{\nu}_i\right)^\top\bs{R}(\bs{\mu}_\mathrm{d}-\bs{\nu}_i) \\
    \textup{s.t. } & \quad \bs{\xi}_0 = \bs{\chi}, \nonumber\\ & \quad \bs{\xi}_{i+1} = \bs{\xi}_i + \bs{\nu}_i\delta_t \quad \forall i\in[0,N],  \nonumber\\
    & \quad h(\bs{\xi}_{i+1},t_{i+1}) \geq \rho h(\bs{\xi}_i,t_i) \quad \forall i \in [0,N]. \nonumber 
\end{align}
\noindent $\bs{R}$ balances the value of various control actions; the time step $\delta_t\in\R_{>0}$ and the increment $N\in\mathbb{Z}^+$ determine the prediction horizon $T=N\delta_t$; the CBF parameter $\rho\in\left(0,1\right)$ affects dynamic robustness; and the initial $\bs{\xi}_0$ is constrained to match the state $\bch$ of the ROM. 

The time-dependent function $h$ accounts for the uncertain dynamic nature of $\C'(\theta, t)$ along the horizon $t \in [0,T]$ and is coarsely modeled using a first-order extrapolation:
\begin{align}
\label{eq: extrapolation}
    h(\bs{\chi}, t) := h_0(\bs{\chi}) + \frac{h_0(\bs{\chi})-h_{\negone}(\bs{\chi})}{t_0 - t_{\negone}}(t-t_0),
\end{align}
% \begin{align}
% \label{eq: extrapolation}
%     h(\bq, t) := h_0(\bq) + \frac{h_0(\bq)-h_{\negone}(\bq)}{t_0 - t_{\negone}}(t-t_0),
% \end{align}
%
\noindent where $h_0$ describes the environment at the current time $t_0\in\re$, and $h_{\negone} : \R^2\times\Sp^1 \rightarrow \R$ is the PSF computed at some previous time $t_{\negone}\in\re$ for $t_{\negone} < t_0$. 

The MPC problem \eqref{eq: NMPC}-\eqref{eq: extrapolation} can be solved with sequential quadratic programming (SQP). The first output in the SQP minimizer $\bs{\nu}_0^*$ is taken as the planned safe control action $\bs{\mu}_\mathrm{p}$. 

\subsection{Real-Time Safety Layer}

To synthesize a safe velocity $\boldsymbol{\mu}_\mathrm{s}$ that renders $\mathcal{C}'(\theta,t_0)$ forward invariant, we filter the planned action $\boldsymbol{\mu}_\mathrm{p}$ from the predictive layer via the \textit{input-to-state safe} (ISSf) CBF-QP:
\begin{align}
 \boldsymbol{\mu}_\mathrm{s} = &\argmin_{\boldsymbol{\mu} \in \mathbb{R}^3} &&\|\boldsymbol{\mu} - \boldsymbol{\mu}_\mathrm{p}\|_2^2 \nonumber\\
    & \quad 
 \mathrm{s.t.} && \!\!\!\nabla h \cdot \boldsymbol{\mu}+ \frac{\partial h}{\partial t}\geq -\alpha h + \frac{1}{\varepsilon}\|\nabla h\|^2, \label{eq: our-Issf}
\end{align}
where  $\frac{1}{\varepsilon}\|\nabla h\|^2$ is a robustness term compensating for the \textit{tracking error} between the ROM and full-order system \cite{compton25a, cohen2024safety}. We denote the safe ROM command by $\boldsymbol{\mu}_s \in \mathbb{R}^3$, and define the corresponding desired velocity as
$\dot{\bch}_\mathrm{s} \coloneqq \boldsymbol{\mu}_\mathrm{s}.$

This layered architecture interprets safety-critical control as certifying the ability of the full-order robotic system to track the velocity commands $\dot{\bch}_\mathrm{s}$ generated by the ROM \eqref{eq: single integrator}. To demonstrate this rigorously, consider the full-order mechanical system with state $(\mathbf{q}, \dot{\mathbf{q}}) \in \mathbb{R}^n \times \mathbb{R}^n$ representing all generalized coordinates and their velocities, governed by:
%
% To achieve this, we employ a layered architecture that interprets safety-critical control as certifying the ability of the full-order robotic system to track centroid velocity commands generated by the ROM \eqref{eq: single integrator}.
%
% To demonstrate this, consider the full-order mechanical system with generalized coordinates $\mathbf{q} \in \mathbb{R}^n$ (with $n \gg 3$) representing all joint angles and velocities $\dot{\mathbf{q}} \in \mathbb{R}^n$, governed by:
\begin{align}
\label{eq: mechanical system}
M(\mathbf{q})\ddot{\mathbf{q}} + C(\mathbf{q}, \dot{\mathbf{q}})\dot{\mathbf{q}} +G(\mathbf{q}) = \mathbf{u},
\end{align}
where $M(\mathbf{q}) \in \mathbb{R}^{n \times n}$, $C(\mathbf{q}, \dot{\mathbf{q}})\in \mathbb{R}^{n \times n}$, $G(\mathbf{q})\in \mathbb{R}^{n}$ denote the inertia matrix, Coriolis matrix, and gravity terms, respectively, and $\mathbf{u} \in \mathbb{R}^n$ is the generalized input. 
The ROM state $\boldsymbol{\chi}$ is obtained from the full configuration via an \textit{output map} $\boldsymbol{\chi} = \varphi(\mathbf{q})$ where $\varphi: \mathbb{R}^n \to \mathbb{R}^2 \times \mathbb{S}^1$ represents the relationship between generalized coordinates and ROM pose \cite{cohen2025safety}. Safety for the full-order system is guaranteed provided the true velocity $\dot{\boldsymbol{\chi}} = \frac{\partial \varphi}{\partial \mathbf{q}}\dot{\mathbf{q}}$ tracks the safe velocity $\dot{\bch}_\mathrm{s}$ sufficiently well, which we establish in the following:

\begin{theorem}\label{eq: main-thm} (Full-Order System Safety) 
Consider the full-order system \eqref{eq: mechanical system} and $T>0$. Let $h:\mathbb{R}^2 \times \mathbb{S}^1 \times [0,T] \rightarrow \mathbb{R}$ be a CBF for the ROM \eqref{eq: single integrator}, let $\boldsymbol{\chi}_\mathrm{s}:[0,T] \rightarrow \mathbb{R}^3$ be a safe velocity satisfying \eqref{eq: our-Issf} for some $\varepsilon >0$, and let $V:\mathbb{R}^n \times \mathbb{R}^n \times \mathbb{R}^3 \times [0,T] \rightarrow \mathbb{R}$ be a tracking control Lyapunov function (CLF) satisfying:
\begin{align} 
    V(\mathbf{q}, \dot{\mathbf{q}}, \dot{\boldsymbol{\chi}}_\mathrm{s}, t) &\geq \beta \|\dot{\boldsymbol{\chi}} - \dot{\boldsymbol{\chi}}_\mathrm{s}\|^2, \label{eq:Lyapunov_bound}\\ 
    \dot{V}(\mathbf{q}, \dot{\mathbf{q}}, \dot{\boldsymbol{\chi}}_\mathrm{s}, \ddot{\boldsymbol{\chi}}_\mathrm{s}, \mathbf{k}, t) &\leq -\lambda V(\mathbf{q}, \dot{\mathbf{q}}, \dot{\boldsymbol{\chi}}_\mathrm{s}, t), \label{eq: clf bound}
\end{align}
for some $\beta, \lambda >0$ and tracking controller $\mathbf{u} = \mathbf{k}(\mathbf{q}, \dot{\mathbf{q}}, \dot{\boldsymbol{\chi}}_\mathrm{s},t)$. Define the barrier candidate $B: \re^n \times \re^n \times [0,T] \rightarrow \re$:
\begin{align}\label{eq: full order CBF}
B(\mathbf{q}, \dot{\mathbf{q}}, t) = h(\varphi(\mathbf{q}), t) - \frac{1}{\mu} V(\mathbf{q}, \dot{\mathbf{q}}, \dot{\boldsymbol{\chi}}_\mathrm{s}, t)
\end{align}
with $\mu>0$. If the following condition holds: 
\begin{align} \label{eq:thm1_bound}
    \lambda \geq \alpha + \frac{\varepsilon \mu}{4 \beta},
\end{align}
then $B$ is a CBF for the full-order system \eqref{eq: mechanical system}, rendering the time-varying set $\mathcal{C}_\mathrm{F}(t) = \{(\mathbf{q}, \dot{\mathbf{q}}) \in \mathbb{R}^{n} \times \mathbb{R}^{n} \, | \,  B(\mathbf{q}, \dot{\mathbf{q}}, t) \geq 0\}$ forward invariant for all $t \in [0,T]$.
\end{theorem}
\begin{proof}
Taking the time derivative of \eqref{eq: full order CBF} and noting that $\dot{\boldsymbol{\chi}} = \frac{\partial \varphi}{\partial \mathbf{q}}\dot{\mathbf{q}}$ gives (suppressing function arguments for brevity):
\begin{align}
    \!\!\dot{B}(\bq, \dq, \bu,t) = \frac{\partial h}{\partial \boldsymbol{\chi}} \frac{\partial \varphi}{\partial \mathbf{q}}\dot{\mathbf{q}} + \frac{\partial h}{\partial t} -  \frac{1}{\mu}\dot{V}(\bq, \dot \bq_s, \ddot \bq_s, \bu, t)
\label{eq: Bdot}
\end{align}
where:
\[
\dot V = \frac{\partial V}{\partial \bq}\,\dq
  + \frac{\partial V}{\partial \dq}\, \ddq
  + \frac{\partial V}{\partial \dot{\boldsymbol{\chi}}_s}\,\ddot{\boldsymbol{\chi}}_s
  + \frac{\partial V}{\partial t},
\]
and \eqref{eq: mechanical system} yields an affine dependence of $\ddq$ on $\bu$. Applying the tracking controller $\bu = \bk(\bq, \dot \bq, \dot{\bch}_s,t)$ satisfying \eqref{eq: clf bound} and lower bounding $\dot{B}$ along the closed loop system yields:

\begin{align}
    \dot{B} &\geq \frac{\partial h}{\partial \bch} \, \left(\dot{\bch} + \dot{\bch}_s- \dot{\bch}_s\right)+ \frac{\partial h}{\partial t} +  \frac{\lambda}{\mu}V\\
    &\geq -\alpha h + \frac{1}{\varepsilon }\|\nabla h\|^2 - \left\|\frac{\partial h}{\partial \bch}\right\| \left\| \dot{\bch}- \dot{\bch}_s\right\| +  \frac{\lambda}{\mu}V \\
    &\geq -\alpha h - \frac{\varepsilon}{4} \left\| \dot{\bch}- \dot{\bch}_s \right\|^2 +  \frac{\lambda}{\mu}V \\
    &\geq -\alpha h - \frac{\varepsilon}{4 \beta} V +  \frac{\lambda}{\mu}V\\
    &= -\alpha B + \frac{1}{\mu}\left(\lambda - \alpha  -\frac{\varepsilon \mu }{4 \beta} \right)V\\
    &\geq -\alpha B. 
\end{align}
% as the acceleration $\ddq$ depends affinely on $\bu$ from \eqref{eq: mechanical system}. 
% such that
% $\dot V$ depends on $\bu$ through the factor
% \[
% \frac{\partial V}{\partial \dq}\,M(\bq)^{-1}B(\bq)\,\bu.
% \]

The first inequality follows from substituting \eqref{eq: clf bound} into \eqref{eq: Bdot} and adding zero. Then \eqref{eq: our-Issf} lower bounds this inequality. Completing the squares and lower bounding yields the third inequality. The fourth is reached with \eqref{eq:Lyapunov_bound}. Using \eqref{eq: full order CBF}, the last equality is reached, which is bounded with \eqref{eq:thm1_bound}.
\end{proof}

\begin{remark}[Verification]
Although Theorem~\ref{eq: main-thm} assumes the existence of a Lyapunov function $V$ and a tracking controller $\mathbf{k}$, explicit expressions are not required to implement our ROM-based safety filter. They serve as tools for verification, i.e., 
one can certify that a given full-order system tracks the ROM commands by providing bounds on $V$ and stability properties of $\bk$.
When such bounds are difficult to obtain, 
(e.g., when $V$ is unknown and $\bk$ is implemented as a black-box or learned module in the existing low-level control architecture), 
Theorem~\ref{eq: main-thm} can still be met in practice by initializing $\alpha$ and $\varepsilon$ in \eqref{eq: our-Issf} at small values to ensure conservative safety, and adjusting until adequate performance is attained.
\end{remark}

% \subsection{Multistage Safety Filter}
% ... % Continuation of methods: MPC + CBF

\section{Multistage Safety Filter: Analysis}\label{sect:experiment}

The previous section established our architecture: PSF synthesis, predictive safe planning, and real-time safety filtering. In this section, we demonstrate how combining these layers in a \textit{multistage} architecture achieves favorable tradeoffs between performance and robustness,
% task optimality and safety robustness, '
compared to either layer operating independently.
%
% This separation targets two objectives simultaneously during collision avoidance: optimality through planning and robustness through instantaneous safety. 
%
We examine these tradeoffs in a detailed Pareto optimality analysis and validate them through experiments on quadruped and humanoid hardware.

\subsection{Pareto Optimality Framework}

We characterize performance and robustness tradeoffs by examining \textit{task optimality} (deviation from desired behavior) and \textit{safety robustness} (minimum safety margin). We formalize these objectives through two metrics evaluated over the horizon $N$, spanning an isolated collision avoidance scenario:
\begin{align*}
    J_{\text{robustness}} &= \min_{i\in[0,N]} h_0(\bch(t_i),t_i), \\
    J_{\text{optimality}} &= \frac{\sum_{i=0}^{N} (\bs{\nu}_i-\bs{\mu}_d)^\top R(\bs{\nu}_i-\bs{\mu}_d)}{J_{\text{ideal}}},
\end{align*}
where $J_\text{robustness}$ is the minimum value of the PSF $h_0$ during the maneuver, and $J_{\text{optimality}}$ is the ratio between the total measured control effort of the maneuver $\nu$ and the minimum ideal effort $J_{\text{ideal}}$. 
The ideal cost $J_{\text{ideal}}$ is computed by simulating an equivalent scenario with \textit{a priori} known obstacles, providing a benchmark for minimum total safe control effort $\nu^*$ for the state-constrained optimal control problem: 
\begin{align*} 
    J_{\text{ideal}} = \sum_{i=0}^{N} (\bs{\nu}^*_i-\bs{\mu}_d)^\top R(\bs{\nu}^*_i-\bs{\mu}_d). 
\end{align*} 
Since these objectives cannot be simultaneously optimized, we introduce a Pareto front to characterize tradeoffs. A solution is \textit{Pareto optimal} if no other solution improves one objective without degrading the other more. The Pareto optimal solution is found by maximizing a\textit{ welfare} function, defined as the weighted sum of utilities, 
\begin{align*}
    \mathbf{z}^{*} =
    \argmax_{\mathbf{z}\in\mathcal{Z}}
    \left(
    a_O u_O(\mathbf{z})
    +
    a_R u_R(\mathbf{z})
    \right),
\end{align*}

\noindent where $\mathbf{z}$ denotes a candidate safety filter architecture, $\mathcal{Z}$ is the set of evaluated architectures, and the utility functions, $u_O(\mathbf{z})$ and  $u_R(\mathbf{z})$, quantify the performance of the method $\mathbf{z}$ along the optimality and robustness metrics, respectively. The scalar weights $a_O, a_R > 0$ specify the relative importance of each objective. The weights are set to $a_O=a_R=0.5$, assigning equal importance to optimality and robustness.

% Fig.\;\ref{fig:pareto} illustrates the theorized position of each controller on the front.

\subsection{Individual Layer Tradeoffs}

When operating independently, each safety layer prioritizes different objectives at the expense of the other, thus occupying extreme positions on the Pareto front.

\subsubsection{Predictive Layer} This layer optimizes a finite-horizon cost, promoting actions that reduce total effort. Long-horizon strategies achieve high optimality by sacrificing immediate safety margins (e.g., side-stepping to avoid obstacles).
\subsubsection{Real-Time Layer}
The CBF-QP enforces instantaneous safety while minimizing deviations from nominal commands. This results in reactive behaviors (e.g., backing away from obstacles), improving robustness by reducing optimality.

\subsubsection{Multistage Architecture} The multistage design feeds the predictive layer command $\bs{\mu}_p$ to the real-time CBF-QP. The predictive layer provides long-horizon intent, while the real-time layer enforces the barrier condition by minimally modifying that intent before execution. This composition preserves the predictive layer's efficiency, while improving instantaneous safety guarantees via the real-time CBF-QP, yielding a Pareto-optimal solution between the two extremes.

\begin{figure}[t]
    \centering    
    \begin{subfigure}[b]{1.0\linewidth}
        \centering    
        \includegraphics[width=\linewidth]{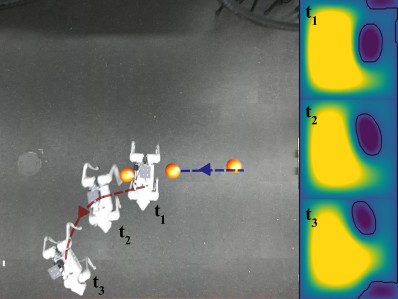}
    \end{subfigure}
    \begin{subfigure}[b]{1.0\linewidth}    
        \centering    
        \includegraphics[width=\linewidth]{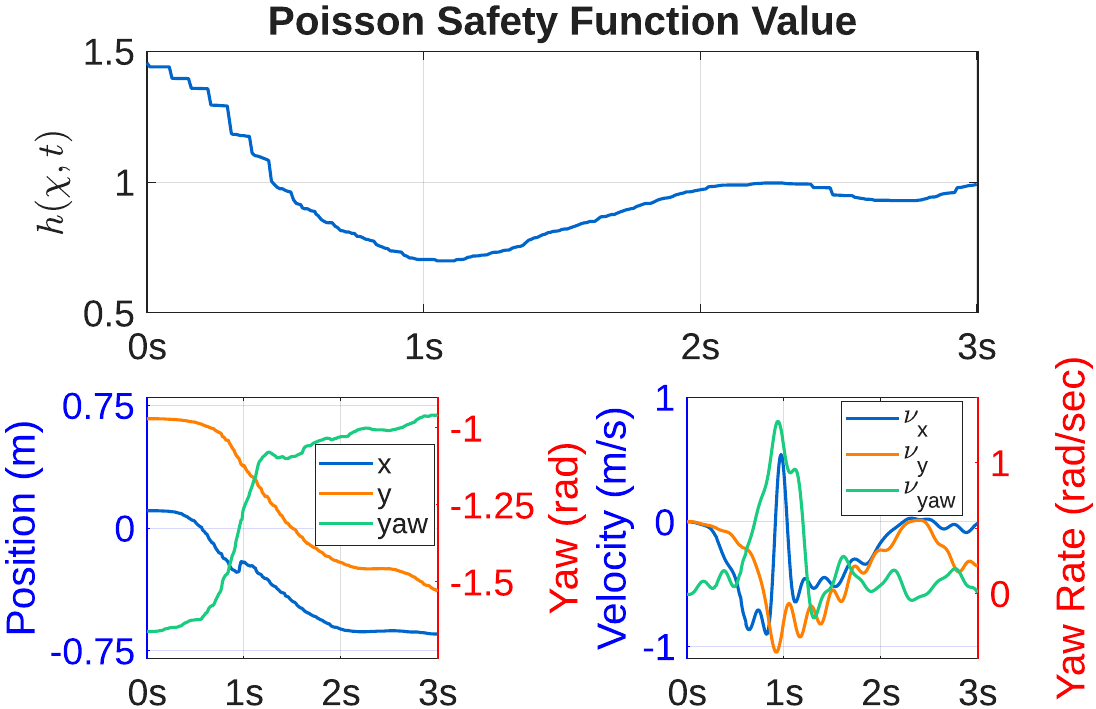}
    \end{subfigure}
   
    \caption{\footnotesize Dynamic Collision Avoidance Trial. \textbf{Top left} Quadruped avoiding a ball from right to left at three times. \textbf{Top right} Robocentric PSFs $h$ at each time instance. \textbf{Middle} The safety function value $h(x,y,\theta,t)$ used to compute $J_{\text{robustness}}$.  \textbf{Bottom Left} Measured position and yaw angle of the quadruped. \textbf{Bottom Right} Filtered velocity commands used to compute $J_{\text{optimality}}$. Footage of experiment trials is presented in the supporting video.}
    \label{fig:roll_instance}
    \vspace{-7mm}
\end{figure}

\begin{table*}[t]
\centering
\caption{\footnotesize Comparison of safety architectures during collision avoidance experiments. Metrics are reported as a mean and 85\% confidence interval (via Student's t-distribution). Within each experimental scenario, the highest lower confidence bound for each metric is in \textbf{bold}.}
\label{tab:ablation_metrics}
\scriptsize 
\setlength{\tabcolsep}{3pt}

\resizebox{\textwidth}{!}{
\begin{tabular}{l cccc cccc cccc} 
    \toprule
    \multirow{2}{*}{\textbf{Scenario}} & \multicolumn{4}{c}{\textbf{Predictive}} & \multicolumn{4}{c}{\textbf{Real-Time}} & \multicolumn{4}{c}{\textbf{Multistage}} \\
    \cmidrule(lr){2-5} \cmidrule(lr){6-9} \cmidrule(lr){10-13}
    & \textbf{Optimality} & \textbf{Robustness} & \textbf{Success \%} & \textbf{Welfare} & \textbf{Optimality} & \textbf{Robustness} & \textbf{Success \%} & \textbf{Welfare} & \textbf{Optimality} & \textbf{Robustness} & \textbf{Success \%} & \textbf{Welfare} \\
    \midrule
    1-Ball (Slow) & \textbf{0.127$\,\pm\,$0.017} & 0.375$\,\pm\,$0.048 & \textbf{100} &  0.466 & 0.065$\,\pm\,$0.005 & 0.673$\,\pm\,$0.020 & 90 & 0.419 & 0.105$\,\pm\,$0.008 & \textbf{0.682$\,\pm\,$0.021} & \textbf{100} & \textbf{0.610}\\
    1-Ball (Fast) & \textbf{0.201$\,\pm\,$0.018} & 0.153$^*$ & 60 & 0.414 & 0.136$\,\pm\,$0.027 & 0.203$\,\pm\,$0.099 & 30 & 0.192 & 0.174$\,\pm\,$0.022 & \textbf{0.235$\,\pm\,$0.043} & \textbf{100} & \textbf{0.485}\\
    2-Ball & \textbf{0.182$\,\pm\,$0.031} & 0.005$\,\pm\,$0.032 & 70 & 0.325 & 0.129$\,\pm\,$0.015 & 0.070$\,\pm\,$0.018 & 40 & 0.211& 0.158$\,\pm\,$0.028 & \textbf{0.087$\,\pm\,$0.014} & \textbf{90} & \textbf{0.377} \\
    3-Ball & \textbf{0.077$\,\pm\,$0.023} & -0.082$\,\pm\,$0.0411 & 70 & 0.085 & 0.058$\,\pm\,$0.011 & \textbf{0.007$\,\pm\,$0.072} & 50 & 0.070 & 0.055$\,\pm\,$0.007 & -0.018$\,\pm\,$0.055& \textbf{80} & \textbf{0.101} \\
    Humanoid & 0.053$\,\pm\,$0.008 & 0.399$\,\pm\,$0.083 & 60 & 0.141 & 0.057$\,\pm\,$0.004 & 0.399$\,\pm\,$0.069 & \textbf{100} & 0.327 & \textbf{0.065$\,\pm\,$0.008} & \textbf{0.530$\,\pm\,$0.048}& \textbf{100} & \textbf{0.505} \\
    \bottomrule
    \vspace{-2mm}
\end{tabular}

} 
\parbox[t]{\textwidth}{\footnotesize$^*$ For these experiments, $h$ could not be accurately measured during the ball's closest approach. Thus, the minimum observed value of $h$ is reported as a conservative upper bound. The true value for each experimental trial was lower, but unmeasurable.}
\vspace{-4mm}
\end{table*}

% \subsection{Demonstrating Pareto Tradeoffs}
\subsection{Experimental Validation}
Our multistage safety filter was validated using a Unitree Go2 quadruped by comparing three controllers: the isolated predictive filter, the isolated real-time safety filter, and the proposed multistage safety filter. The environment was perceived with two onboard LiDARs. Experiments comprised scenarios where one, two, and three dodgeballs were rolled at the robot. For one-ball experiments, a ball was rolled toward the robot at both 1.0~m/s (slow) and 1.5~m/s (fast). The two-ball experiment introduced a second ball from the opposing side, both at 1.0~m/s. Lastly, three-ball experiments added a ball approaching the front of the robot, with all three rolling at 1.5~m/s. Fig.\;\ref{fig:roll_instance} shows a representative experiment.

Table \ref{tab:ablation_metrics} reports the optimality and robustness metrics across all experiments. Table \ref{tab:ablation_metrics} also reports the avoidance success rate and welfare. Utility associated with each metric is computed by normalizing the values and taking the lower bound of its 85\% confidence interval. While $u_O$ is computed only from successful trials, $u_R$ incorporates failures by assigning trials a normalized robustness score of 0. 

The multistage safety filter consistently achieved the highest success rates and welfare scores for all experiments, supporting its Pareto optimality. The Pareto front is visualized in Fig. \ref{fig:avoidance_all} for the slow ball scenario using Hotelling's T-squared distribution. The predictive filter achieved high optimality at the cost of robustness. The real-time safety filter had low optimality with high robustness. A state-constrained MPC baseline was also evaluated for this first scenario; however, it lacked the necessary robustness for dynamic collision avoidance, failing all 10 trials. Consequently, the performance metrics could not be meaningfully computed. 

As avoidance difficulty increased, failures became more frequent. The predictive filter fails due to its low control rate, limiting its robustness in dynamic environments. Meanwhile, failures of the real-time filter arise from its myopic nature in conjunction with model mismatch (the single integrator model does not account for input limits). At higher speeds, obstacles exceed input limits and collide with the robot. 

Finally, because our ROM-based framework does not make assumptions on the underlying full-order model,  the control architecture can be deployed on additional robotic platforms. Thus,  we implemented it on the Unitree G1 humanoid robot using only the built-in LiDAR for perception. A yoga ball was rolled toward the humanoid robot at 1.1~m/s, and the multistage safety filter again achieved the highest success rate and welfare score (Table \ref{tab:ablation_metrics}), verifying the Pareto optimality tradeoffs persist across robotic platforms. Our multistage architecture also outperformed the other two methods in optimality and robustness metrics. 

% The Pareto front is visualized for the slow single ball experiment. For each avoidance, $J_{\mathrm{robustness}}$ (x-axis) and $J_{\mathrm{optimality}}$ (y-axis) is plotted, along with the joint confidence ellipse of the population mean using the Hotteling's T-squared distribution at 85\% confidence level (Fig.\;\ref{fig:avoidance_all}). The ellipse shows an inverse trend, showing the expected performance and robustness tradeoffs. The predictive layer tends to be more optimal but less robust, and the real-time layer has the lowest optimality, but is much more robust.\\
\begin{figure}
    \centering
    \includegraphics[width=0.75\linewidth]{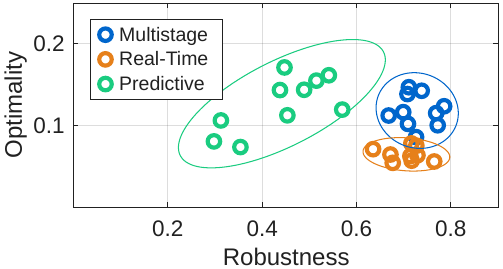}
    \vspace{-1mm}
    \caption{\footnotesize Analysis of Avoidance Experiment. Evaluated optimality and robustness metrics of collision avoidance with ball speeds of 1.0~m/s, with a confidence ellipse of the population mean. The results show tradeoffs in the metrics, forming an approximation of a Pareto front.}
    \label{fig:avoidance_all}
    \vspace{-4mm}
\end{figure}

\begin{figure}[]
    \centering
    \includegraphics[width=1.0\linewidth]{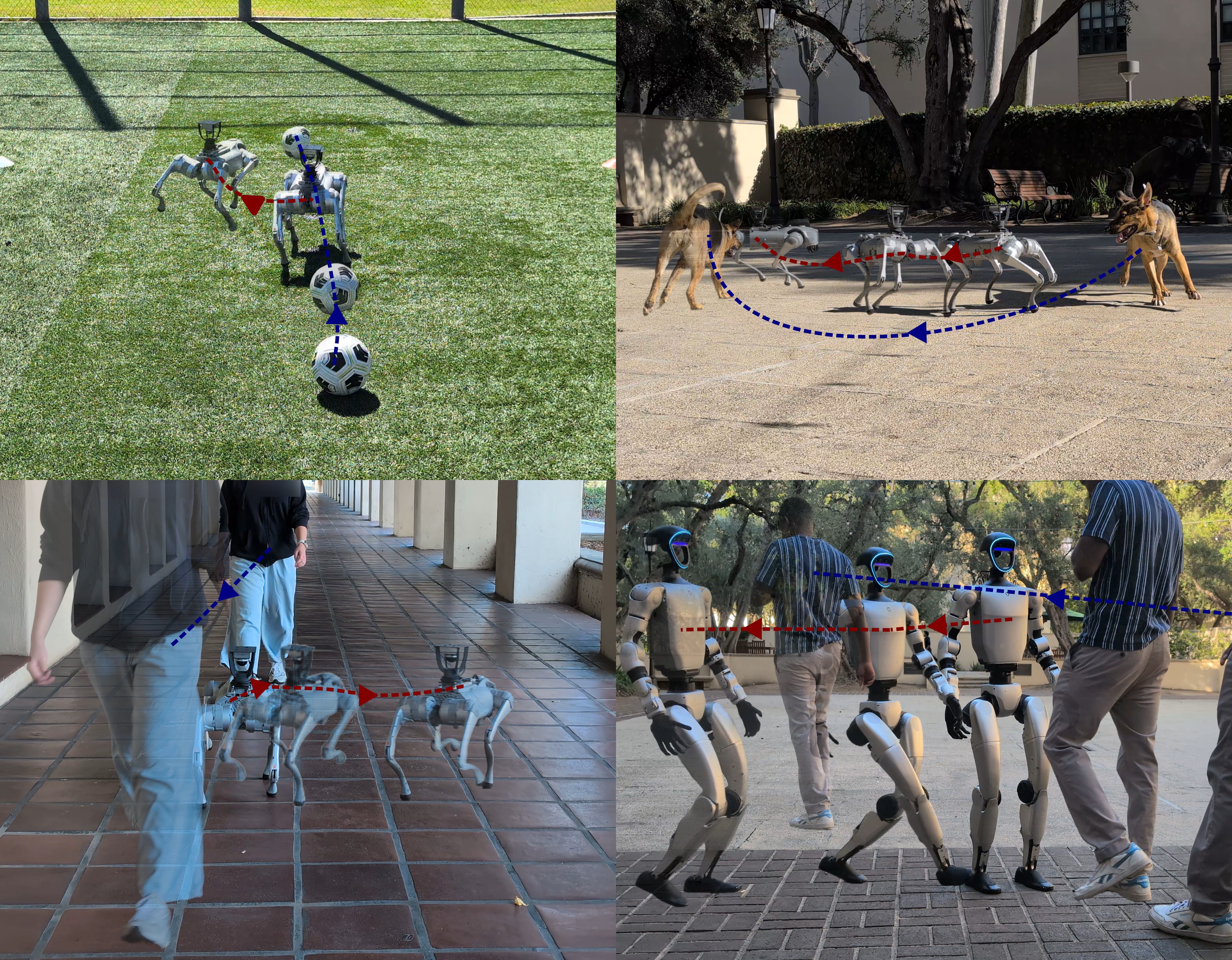}
    \caption{\footnotesize Collision Avoidance in Practical Settings. \textbf{Top Left} Quadruped avoids a soccer ball on a turf field. \textbf{Top Right} Quadruped avoids a persuant dog. \textbf{Bottom Left} Quadruped avoids a pedestrian. \textbf{Bottom Right} Humanoid avoids a human. Obstacle trajectories are indicated in blue, while avoidance trajectories are shown in red. Footage is available in the supplemental video.}
    \label{fig:generalized}
    \vspace{-7mm}
\end{figure}

\subsection{Generality of Multistage Safety Architecture}

To demonstrate the generality of our safety architecture across various environments and obstacles, we deployed it in multiple unstructured real-world scenarios. Examples are depicted in Fig.\;\ref{fig:generalized}. Fig.\;\ref{fig:generalized} (top left) shows the Unitree Go2 quadruped successfully avoiding a rolling ball on a slippery turf field. Fig.\;\ref{fig:generalized} (top right) shows the quadruped safely avoiding an approaching dog, (bottom left) shows safe collision avoidance as the quadruped and human approach one another, and (bottom right) shows the humanoid successfully avoiding collisions with an approaching person, demonstrating applicability across different dynamic obstacles.

\section{CONCLUSIONS}\label{sect:conclusion}

We have presented a multistage architecture for dynamic collision avoidance on multiple robotic platforms. We provided an occupancy mapping algorithm to synthesize safety constrainta via PSFs, which we incorporated into two safety layers: 1) a predictive MPC layer, and 2) a real-time CBF-QP safety filter.
% We have presented a multistage architecture for dynamic collision avoidance. We provided an occupancy mapping algorithm to incorporate perception-in-the-loop. Using this environment occupancy information, we synthesized safety constraints via PSFs, which we incorporated into two safety layers: 1) a predictive MPC layer, and 2) a real-time CBF-QP safety filter.
%
% The predictive layer's outputs served as nominal inputs to the real-time fi
%
We validated the optimality and robustness of this architecture via Pareto analysis on quadruped and humanoid platforms in dynamic environments.
Future work includes extensions to high-dimensional systems, multi-robot systems, and integration of  semantic scene understanding.

\bibliographystyle{ieeetr}
\bibliography{main-GB, erina, cohen}

% \input{appendix}

% \end{thebibliography}

\end{document}